\documentclass[journal]{IEEEtran}

\usepackage{graphics,color}
\usepackage{epsfig}
\usepackage{graphicx}
\usepackage{amsgen}
\usepackage{amsmath}
\usepackage{amssymb}
\usepackage{amstext}
\usepackage{amsbsy}
\usepackage{amsopn}
\usepackage{amsfonts}
\usepackage{amsthm}
\usepackage{url}
\usepackage{bm}
\usepackage{hyperref}

\newcommand{\etal}{\mbox{\em et al.\ }}

\newcommand{\Ad}{\mbox{\rm Ad}}

\newcommand{\tr}{\mbox{\rm tr}}    % tr (=trace) in roman for equations
\newfont{\gothic}{eufm10 scaled\magstep0}
\newcommand{\gothsl}{\mbox{\gothic sl}}

\newcommand{\rr}{\mbox{$\mathbb R$}}

\newtheorem{theorem}{Theorem}
\newtheorem{lemma}{Lemma}

\newtheorem{proposition}{Proposition}

\newtheorem{definition}{Definition}
\newtheorem{assumption}{Assumption}

\newtheorem{remark}{Remark}

\usepackage{multirow}
\usepackage[left=0.71in,top=0.94in,right=0.71in,bottom=1.18in]{geometry}
\begin{document}

\title{Feature-based Recursive Observer Design for Homography Estimation}

\author{Minh-Duc Hua, Jochen Trumpf, Tarek Hamel, Robert Mahony, and Pascal Morin
\thanks{M.-D. Hua and T. Hamel are with I3S UNS--CNRS, Nice-Sophia Antipolis, France.
  email: {\tt hua(thamel)@i3s.unice.fr}.}% <-this % stops a space
\thanks{J. Trumpf and R. Mahony are with the Research School of Engineering,
Australian National University, Canberra, Australia. email: {\tt
    Robert.Mahony(Jochen.Trumpf)@anu.edu.au}.}% <-this % stops a space
\thanks{P. Morin is with ISIR UPMC--CNRS, Paris, France. email:
{\tt morin@isir.upmc.fr}.}% <-this % stops a space
}

\date{}

% The paper headers
%\markboth{IEEE TRANSACTIONS ON ROBOTICS, VOL. , NO. , MONTH YEAR}%
%{Shell \MakeLowercase{\textit{et al.}}: Bare Demo of IEEEtran.cls for Journals}

\maketitle

\begin{abstract}
This paper presents a new algorithm for online estimation of a sequence of homographies applicable to image sequences obtained from robotic vehicles equipped with vision sensors.
The approach taken exploits the underlying Special Linear group structure of the set of homographies along with gyroscope measurements and direct point-feature correspondences between images to develop temporal filter for the homography estimate.
Theoretical analysis and experimental results are provided to demonstrate the robustness of the proposed algorithm.
The experimental results show excellent performance even in the case of very fast camera motion (relative to frame rate), severe occlusion, and in the presence of specular reflections.
\end{abstract}

\nocite{HuaCDC2011}

%%%%%%%%%%%%%%%%%%%%%%%%%%%%%%%%%%%%%%%%%%%%%%%%%%%%%%%%%%%%%%%%%%%%%%%%%%%%%%%%
\section{INTRODUCTION}
%Built environments consist primarily of planar or near planar surfaces.
When a robotic vehicle equipped with a vision sensor is manoeuvering in a built environment, consisting primarily of planar or near planar surfaces, then the nature of the environment can be exploited in the vision processing algorithms.
Different images of the same planar surface are related by homography mappings, and homographies have been used extensively in robotic applications as a vision primitive.
Homography-based algorithms have been used for estimation of the rigid-body pose of a vehicle equipped with a camera \cite{2005_Wang_iros,RM_2007_Mufti_dicta,2008_Scaramuzza_TRO}.
Navigation of robotic vehicles has been developed based on homography sequences \cite{2005_Fang_TSMC,2007_Fraundorfer_iros,2014_Hua_ifac} and one of the most successful visual servo control paradigms uses homographies  \cite{MalChaBou1999,2009_Malis.icra}.
Homography-based methods are particularly well suited to navigation of unmanned aerial vehicles \cite{2007_Caballero_icra,2010_Mondragon_icra,2011_de_Plinval_icra} where the ground terrain is viewed from a distance and the relief of surface features is negligible compared to the vehicles distance from the scene.

Computing homographies from point correspondences has been extensively studied in the last fifteen years \cite{Hartley2003}.
The quality of the homography estimate obtained depends heavily on the number and quality of the data points used in the estimation as well as the algorithm employed.
For a  well textured scene, the state-of-the-art methods can provide high quality homography estimates at the cost of significant computational effort (see \cite{Mei2008} and references therein).
For a scene with poor texture, and consequently few reliable feature point correspondences, existing homography estimation algorithms perform poorly.
Robotic vehicle applications, however, provide temporal sequences of images and it seems natural to exploit the temporal correlation rather than try to compute individual raw homographies for each pair of frames.
Zhang \etal \cite{2007_Zhang_PR} used image flow computed from a pair of images to compute the relative homography, although this method still only considers isolated pairs of images.
In recent work by the authors \cite{2009_Malis.icra,2012_Mahony.IJC} a nonlinear observer \cite{Bonnabel,toappear_Lageman.TAC} for homography estimation was proposed based on the group structure of the set of homographies, the Special Linear group $\mathrm{SL}(3)$ \cite{Benhimane-IJRR-07}.
This observer uses velocity information to interpolate across a sequence of images and improve the individual homography estimates.
The observer proposed in \cite{2009_Malis.icra,2012_Mahony.IJC} still requires individual image homographies to be computed for each image, which are then smoothed using filtering techniques.
Although this approach \cite{2009_Malis.icra,2012_Mahony.IJC} provides improved homography estimates, it comes at the cost of running both a classical homography algorithm as well as the temporal filter algorithm, and only functions if each pair of images has sufficient data available to compute a raw homography.

In this paper, we consider the question of deriving an observer for a sequence of image homographies that takes image point feature correspondences directly as input.
The proposed approach takes a sequence of images associated with a continuous variation of the reference image, the most common case being a moving camera viewing a planar scene, a situation typical of robotic applications.
The proposed nonlinear observer is posed on the Special Linear group $\mathrm{SL}(3)$, that is in one-to-one correspondence with the group of homographies \cite{Benhimane-IJRR-07}, and uses velocity measurements to propagate the homography estimate and fuse this with new data as it becomes available \cite{2009_Malis.icra,2012_Mahony.IJC}.
A key advance on prior work by the authors is the formulation of a point feature innovation for the observer that incorporates point correspondences directly in the observer without requiring reconstruction of individual image homographies.
This saves considerable computational resources and makes the proposed algorithm suitable for embedded systems with simple point tracking software. In addition, the algorithm is well posed even when there is insufficient data for full reconstruction of a homography.
For example, if the number of corresponding points between two images drops below four it is impossible to algebraically reconstruct an image homography and the existing algorithms fail \cite{Hartley2003}.
In such situations, the proposed observer will continue to operate, incorporating what information is available and relying on propagation of prior estimates where necessary.
Finally, even if a homography can be reconstructed from a small set of feature correspondences, the estimate is often unreliable and the associated error is difficult to characterize.
The proposed algorithm integrates information from a sequence of images, and noise in the individual feature correspondences is filtered through the natural low-pass response of the observer, resulting in a highly robust estimate.
As a result, the authors believe that the proposed observer is ideally suited for estimation of homographies based on small windows of image data associated with specific planar objects in a scene, poorly textured scenes,  and real-time implementation; all of which are characteristic of requirements for homography estimation in robotic vehicle applications.

The paper is organized into seven sections including the introduction and the concluding sections.
Section~\ref{sec:Background} presents a brief recap of the Lie group structure of the set of homographies.
In Section~\ref{sec:Observer}, based on a recent advanced theory for nonlinear observer design \cite{MahonyNOLCOS13}, a nonlinear observer on $\mathrm{SL(3)}$ is proposed using direct measurements of known inertial directions and the knowledge of the group velocity.
A rigourous stability analysis is provided in this section.
In Section~\ref{sec:application}, the homography and associated homography velocity are related to rigid-body motion of the camera and two observers are derived for the case where only the angular velocity of the camera is known, a typical scenario in robotic applications.
An initial constructed example in Section~\ref{section:experiment} confirms the robust performance of the proposed algorithm on real image data. Section~\ref{sec:experimentII} provides an application of our approach to a real world problem in image stabilization.
Two videos, showing the results of Sections~\ref{section:experiment} and \ref{sec:experimentII}, respectively, are provided as supplementary material.
The results show excellent performance even in the case of very fast camera motion (relative to frame rate), severe occlusion, and in the presence of specular reflections.

%A preliminary version of the results in this paper was presented at the IEEE CDC 2011 \cite{HuaCDC2011}.

%%%%%%%%%%%%%%%%%%%%%%%%%%%%%%%%%%%%%%%%%%%%%%%%%%%%%%%%%%%%%
\section{Preliminary material}
\label{sec:Background}
\subsection{Camera projections} \label{Projection}
Visual data is obtained via a projection of an observed scene onto the camera image surface. The projection is parameterised by two sets of parameters: intrinsic (``internal'' parameters of the camera such as the focal length, the pixel aspect ratio, etc.) and extrinsic (the pose, i.e. the position and orientation of the camera). Let ${\cal \mathring{A}}$ (resp.~${\cal A}$) denote projective coordinates for the image plane of a camera $\mathring{A}$ (resp.~$A$), and let $\{\mathring{A}\}$ (resp.~$\{A\}$) denote its (right-hand) frame of reference. Let $\xi \in \mathbb{R}^3$ denote the position of the frame $\{A\}$ with respect to $\{\mathring{A}\}$ expressed in $\{\mathring{A}\}$.  The orientation of the frame $\{A\}$ with respect to $\{\mathring{A}\}$, is given by a rotation matrix $R \in\mathrm{SO(3)}$, where $R\colon \{A\} \rightarrow \{\mathring{A}\}$ as a mapping. The pose of the camera determines a rigid body transformation from $\{A\}$ to $\{\mathring{A}\}$ (and visa versa). One has
\begin{eqnarray}\label{:eq:rigid}
\mathring{P} = RP+\xi
\end{eqnarray}
as a relation between the coordinates of the same point in the reference frame ($\mathring{P} \in \{\mathring{A}\}$) and in the current frame ($P \in \{A\}$). The camera internal parameters, in the commonly used approximation \cite{Hartley2003}, define a  $3\times 3$ matrix $K$ so that we can write\footnote{Most statements in projective geometry involve equality up to a multiplicative constant denoted by $\cong$.}:
\begin{equation}\label{:eq:proj-1}
\mathring{p}\cong K\mathring{P}, \quad p \cong KP,
\end{equation}
where $p \in {\cal A}$ is the image of a point when the camera is aligned with frame $\{A\}$, and can be written as $(x,y,w)^\top$ using the homogeneous coordinate representation for that 2D image point. Likewise, $\mathring{p} \in {\cal \mathring{A}}$ is the image of the same point viewed when the camera is aligned with frame $\{\mathring{A}\}$.

If the camera is calibrated (the intrinsic parameters are known), then all quantities can be appropriately scaled and the equation is written in a simple form:
\begin{eqnarray}\label{:eq:camera}
\mathring{p} \cong \mathring{P},\quad p \cong P.
\end{eqnarray}

%~~~~~~~~~~~~~~~~~~~~~~~~~~~~~~~~~~~~~~~~~~~~~~~~~~~~~~~~~~~~~~~
\subsection{Homographies}

Since homographies describe image transformations of {\em planar} scenes, we begin by fixing a plane that contains the points of interest (target points).

\begin{assumption}
Assume a calibrated camera and that there is a planar surface $\pi$ containing a set of $n$ target points ($n \geq 4$) so that
\[
\Pi =\left\{ \mathring{P} \in R^3: \mathring{\eta}^\top \mathring{P} - \mathring{d} = 0\right\},
\]
where $\mathring{\eta}$ is the unit normal to the plane expressed in $\{\mathring{A}\}$ and $\mathring{d}$ is the distance of the plane to the origin of $\{\mathring{A} \}$.
\end{assumption}
From the rigid-body relationships~\eqref{:eq:rigid}, one has $P=R^\top\mathring{P}-R^\top\xi$. Define $\zeta= -R^\top \xi$.  Since all target points lie in a single planar surface one has
\begin{eqnarray}
P_i=R^\top\mathring{P}_i+\frac{\zeta{\mathring{\eta}}^\top}{\mathring{d}}\mathring{P}_i, \quad i = \{1,
\ldots, n\},
\end{eqnarray}
and thus, using~\eqref{:eq:camera}, the projected points obey
\begin{eqnarray}
p_i\cong \left(R^\top+\frac{\zeta{\mathring{\eta}}^\top}{\mathring{d}}\right)\mathring{p}_i,\quad i = \{1,
\ldots, n\}.
\end{eqnarray}
The projective mapping $H: {\cal A}\rightarrow {\cal \mathring{A}}$, $H :\cong \left(R^\top+\frac{\zeta{\mathring{\eta}}^\top}{\mathring{d}}\right)^{-1}$ is termed a homography and it relates the images of points on the plane $\Pi$ when viewed from two poses defined by the coordinate systems ${\cal A}$ and ${\cal
\mathring{A}}$, respectively. It is straightforward to verify that the homography $H$ can be written as follows:
\begin{equation}
H \cong\left(R  + \frac{ \xi {\eta}^\top}{d}\right),
\label{homog}\end{equation}
where $\eta$ is the normal to the observed planar surface expressed in the frame $\{A\}$ and $d$ is the orthogonal distance of the plane to the origin of $\{A\}$. One can verify that \cite{Benhimane-IJRR-07}:
\begin{align}
\eta &=R^\top\mathring{\eta} \\
d &= \mathring{d}-\mathring{\eta}^\top\xi=\mathring{d}+\eta^\top \zeta. \label{eq:dB}
\end{align}
The homography matrix contains the pose information $(R, \xi)$ of the camera from the frame $\{A\}$ (termed current frame) to the frame $\{\mathring{A}\}$ (termed reference frame). However, since the relationship between the image points and the homography is a projective relationship, it is only possible to determine $H$ up to a scale factor (using the image points relationships alone).

\subsection{Homography versus element of the Special Linear Group $\mathrm{SL}(3)$}

Recall that the Special Linear group $\mathrm{SL}(3)$ is defined as the set of all real valued $3 \times 3$ matrices with unit determinant:
\[
\mathrm{SL}(3) = \{ H \in \mathbb{R}^{3\times 3}\;|\; \det{H} = 1 \}.
\]
The Lie-algebra $\gothsl(3)$ for $\mathrm{SL}(3)$ is the set of matrices with trace equal to zero: $\gothsl(3) = \{ X \in \rr^{3 \times 3} \;|\; \tr(X) = 0 \}$.  The adjoint operator is a mapping $\Ad : \mathrm{SL}(3) \times \gothsl(3) \rightarrow \gothsl(3)$ defined by:
\[
\Ad_H X = HXH^{-1}, \quad H \in \mathrm{SL}(3) , X \in \gothsl(3).
\]

Since a homography matrix $H$ is only defined up to scale then any homography matrix is associated with a unique matrix $\bar{H} \in \mathrm{SL}(3)$ by re-scaling
\begin{equation}
\bar{H} = \frac{1}{\det(H)^{\frac{1}{3}}} H \label{H-sl3}
\end{equation} such that $\det(\bar{H}) = 1$.
Every such matrix $\bar{H} \in \mathrm{SL}(3)$ occurs as a homography.
Moreover, the map
\begin{align*}
  w\colon &\mathrm{SL}(3)\times\mathbb{P}^{2}\longrightarrow\mathbb{P}^{2},\\
  &\quad\quad(H,p)\mapsto w(H,p)\cong \frac{Hp}{\vert Hp\vert}
\end{align*}
is a group action of $\mathrm{SL}(3)$ on the projective space $\mathbb{P}^2$ since
\begin{align*}
w(H_1, w(H_2,p)) &= w(H_1 H_2,p),\; w(I, p) = p,
\end{align*}
where $H_1, H_2$ and $H_1 H_2 \in \mathrm{SL}(3)$ and $I$ is the identity matrix, the unit element of $\mathrm{SL}(3)$. The geometrical meaning of the above property is that the 3D motion of the camera between views ${\cal A}_0$ and ${\cal A}_1$, followed by the 3D motion between views ${\cal A}_1$ and ${\cal A}_2$ is the same as the 3D motion between views ${\cal A}_0$ and ${\cal A}_2$. As a consequence, we can think of homographies as described by elements of $\mathrm{SL}(3)$.

Since any homography is defined up to a scale factor, we assume from now on that $H \in \mathrm{SL}(3)$:
\begin{equation}
\label{eq:homog-decomp}
H = \gamma\left(R  + \frac{ \xi {\eta}^\top}{d}\right).
\end{equation}
There are numerous approaches for determining $H$, up to this scale factor, cf.~for example \cite{MaVa2007}. Note that direct computation of the determinant of $H$ in combination with the expression of $d$ \eqref{eq:dB} and using the fact that $\det(H)=1$, shows that $\gamma=(\frac{d}{\mathring{d}})^{\frac{1}{3}}$.

Extracting $R$ and $\frac{\xi}{d}$ from $H$ is in general quite complex \cite{Benhimane-IJRR-07,zhang,weng,FaugerasLustman} and is beyond the scope of this paper.

\section{Nonlinear observer design on $\mathrm{SL(3)}$ based on direct measurements}
\label{sec:Observer}
In this section, the design of an observer for $H\in \mathrm{SL(3)}$ is based on a recent theory for nonlinear observer design directly on the output space \cite{MahonyNOLCOS13}.

\subsection{Kinematics and measurements}
Consider the kinematics of $\mathrm{SL(3)}$ given by
\begin{equation}\label{:eq:dotH}
\dot H = F(H,U) := H U,
\end{equation}
with $U \in \mathfrak{\mathfrak{sl}}(3)$ the group velocity. Assume that $U$ is measured. Furthermore, we dispose of a set of $n$ measurements $p_i \in \mathbb{P}^2$ in the body-fixed frame:
\begin{equation}\label{p_i}
p_i= h(H,\mathring{p}_i) :=\frac{H^{-1}\mathring{p}_i}{|H^{-1}\mathring{p}_i|},\quad i=\{1 \ldots n\},
\end{equation}
where $\mathring{p}_i \in \mathbb{P}^2$ are constant and known. For later use, define
\[
\mathring{p} :=(\mathring{p}_1,\cdots,\mathring{p}_n),\quad p:=(p_1,\cdots,p_n).
\]

\begin{definition}
A set ${\mathcal M}_n$ of $n\geq 4$ vector directions $\mathring{p}_i \in \mathbb{P}^2$, with $i=\{1 \ldots n\}$, is called \emph{consistent}, if it contains a subset ${\mathcal M}_4 \subset {\cal M}_n$ of $4$ constant vector directions such that all vector triplets in ${\mathcal M}_4$ are linearly independent.
\end{definition}
This definition implies that if the set ${\cal M}_n$ is consistent then, for all $\mathring{p}_i \in {\cal M}_4$ there exists a unique set of three non vanishing scalars $b_j \neq 0$ ($j \neq i$) such that
\[
\mathring{p}_i= \frac{y_i}{|y_i|} \mbox{ where } y_i=\sum_{j=1 (j\neq i)}^4 b_j \mathring{p}_j.
\]

We verify that $\mathrm{SL(3)}$ is a symmetry group with group actions $\phi: \mathrm{SL(3)} \times \mathrm{SL(3)} \longrightarrow \mathrm{SL(3)}$, $\psi: \mathrm{SL(3)}\times \mathfrak{sl}(3) \longrightarrow \mathfrak{sl}(3)$ and $\rho : \mathrm{SL(3)}\times \mathbb{P}^2 \longrightarrow \mathbb{P}^2$ defined by
\[
\begin{array}{lcl}
\phi(Q,H) &:=& H Q,\\
\psi(Q,U) &:=& Ad_{Q^{-1}} U = Q^{-1} U Q,\\
\rho(Q,p) &:=& \frac{Q^{-1} p}{|Q^{-1} p|}.
\end{array}
\]
Indeed, it is straightforward to verify that $\phi$, $\psi$, and $\rho$ are {\em right group actions} in the sense that
$\phi(Q_2, \phi(Q_1,H)) = \phi(Q_1Q_2,H)$, $\psi(Q_2, \psi(Q_1,U)) = \psi(Q_1Q_2,U)$, and $\rho(Q_2, \rho(Q_1,p)) = \rho(Q_1Q_2,p)$, for all $Q_1,Q_2,H \in \mathrm{SL(3)}$, $U\in \mathfrak{sl}(3)$, and $p\in \mathbb{P}^2$. Clearly,
\[
\rho(Q, h(H,\mathring{p}_i)) = \frac{Q^{-1} \frac{H^{-1} \mathring p_i}{|H^{-1} \mathring{p}_i|}}{\big|Q^{-1} \frac{H^{-1} p_i}{|H^{-1} \mathring{p}_i|}\big|}=h(\phi(Q,H),\mathring{p}_i),
\]
and
\[
\begin{array}{lcl}
d\phi_Q(H)[F(H,U)] &=& HU Q = (HQ) (Q^{-1} UQ) \\
&=& F(\phi(Q,H), \psi(Q, U)).
\end{array}
\]
Thus, the kinematics are {\em right equivariant} (see \cite{MahonyNOLCOS13}).

\subsection{Observer design}
Let $\hat H \in \mathrm{SL(3)}$ denote the estimate of $H$. Define the right group error $E = \hat H H^{-1} \in \mathrm{SL(3)}$
and the output errors $e_i \in \mathbb{P}^2$:
\begin{equation}\label{:eq:outputErrors}
e_i := \rho(\hat H^{-1}, p_i) = \frac{\hat H p_i}{|\hat H p_i|} =  \frac{E \mathring{p}_i}{|E \mathring{p}_i|}.
\end{equation}
Inspired by \cite{MahonyNOLCOS13}, the proposed observer takes the form
\begin{equation}\label{:eq:ObserverGeneral}
\dot {\hat H} = \hat H U - \Delta(\hat H, p) \hat H,
\end{equation}
where $\Delta(\hat H, p) \in \mathfrak{sl}(3)$ is the innovation term to be designed and must be {\em right equivariant} in the sense that for all $Q \in \mathrm{SL(3)}$:
\[
\Delta(\hat H Q, \rho(Q, p_1), \cdots, \rho(Q, p_n)) = \Delta(\hat H, p_1, \cdots,p_n).
\]
Interestingly, if the innovation term $\Delta$ is equivariant, the dynamics of the canonical error $E$ is autonomous and given by \cite[Th. 1]{MahonyNOLCOS13}:
\begin{equation}\label{:eq:ErrorObserverGeneral}
\dot E = -\Delta(E, \mathring{p}) E.
\end{equation}

In order to determine $\Delta(\hat H,p)$, we need a {\em non-degenerate right-invariant} cost function. To this purpose, let us first define
individual {\em degenerate right-invariant costs} at $\mathring{p}_i$ on the output space $\mathbb{P}^2$ as follows:
\begin{align*}
  \mathcal{C}_{\mathring{p}_i}^i \colon &\mathrm{SL(3)} \times \mathbb{P}^2 \longrightarrow \mathbb{R}^+,\\
  &\quad\,\,(\hat H,p_i) \,\,\,\mapsto \,{\mathcal{C}}_{\mathring{p}_i}^i(\hat H,p_i) := \frac{k_i}{2}\left|\frac{\hat H p_i}{|\hat H p_i|} - \mathring{p}_i\right|^2
\end{align*}
with $k_i$ positive numbers. One verifies that ${\mathcal{C}}_{\mathring{p}_i}^i(\hat H,p_i)$ are right-invariant in the sense that ${\mathcal{C}}_{\mathring{p}_i}^i (\hat H Q, \rho(Q, p_i)) = {\mathcal{C}}_{\mathring{p}_i}^i(\hat H, p_i)$ for all $Q \in \mathrm{SL(3)}$. The costs ${\mathcal{C}}_{\mathring{p}_i}^i(\hat H,p_i)$ are degenerate since by taking $p_i=\mathring{p}_i$ there exists an infinity of $\hat H$ such that ${\mathcal{C}}_{\mathring{p}_i}^i(\hat H,\mathring p_i) = 0$.

Then, the aggregate cost is defined as the sum of all the individual costs as follows:
\begin{equation}\label{:eq:aggregateCost}
\begin{split}
  \!\!\!\!{\mathcal{C}}_{\mathring{p}} \colon &\mathrm{SL(3)} \!\times\! (\mathbb{P}^2 \times \cdots \times \mathbb{P}^2) \!\longrightarrow \mathbb{R}^+,\\
  &\quad\quad\,(\hat H,p)\,\,\,\,\mapsto {\mathcal{C}}_{\mathring{p}}(\hat H,p) :=  \!\sum_{i=1}^n\frac{k_i}{2}\left|\frac{\hat H p_i}{|\hat H p_i|} - \mathring{p}_i\right|^2
\end{split}
\end{equation}
It is straightforward that ${\mathcal{C}}_{\mathring{p}}(\hat H,p)$ is right-invariant.
According to \cite[Lem. 3]{MahonyNOLCOS13}, the aggregate cost is {\em non-degenerate} if
\[\bigcap_{i=1}^n \mathrm{stab}_{\rho}(\mathring{p}_i) = \{I\}\]
where the stabilizer $\mathrm{stab}_{\rho}(p)$ of an element $p\in \mathbb{P}^2$ is defined by
$\mathrm{stab}_{\rho}(p) = \{Q \in \mathrm{SL(3)} \mid \rho(Q,p) = p\}$. In fact, $\bigcap_{i=1}^n \mathrm{stab}_{\rho}(\mathring{p}_i) = \{I\}$ is equivalent to $\bigcap_{i=1}^n \mathfrak{s}_i = \{0\}$, where $\mathfrak{s}_i = \mathrm{ker} (\mathrm{d}\rho_{\mathring{p}_i}(I))$ is the Lie-algebra associated with $\mathrm{stab}_{\rho}(\mathring{p}_i)$ (see \cite{MahonyNOLCOS13}).

\begin{lemma}\label{:lem:aggregateCost}
Assume that the set ${\mathcal M}_n$ of the measured directions $\mathring{p}_i$ is consistent. Then, the aggregate cost ${\mathcal{C}}_{\mathring{p}}(\hat H,p)$ defined by \eqref{:eq:aggregateCost} is non-degenerate. As a consequence, $(I,\mathring{p})$ is a global minimum of the aggregate cost ${\mathcal{C}}_{\mathring{p}}$.
\end{lemma}

\begin{proof}
One computes the derivative
\begin{equation} \label{:eq:drho}
\begin{array}{lcl}
\mathrm{d}\rho_{\mathring{p}_i}(H)[HU] &\!\!\!\!=\!\!\!\!& \mathrm{d} \left(\frac{H^{-1}\mathring{p}_i}{|H^{-1}\mathring{p}_i|}\right) [HU]\\
&\!\!\!\!=\!\!\!\!& \left(I - \frac{(H^{-1}\mathring{p}_i)(H^{-1}\mathring{p}_i)^\top}{|H^{-1}\mathring{p}_i|^2} \right) \frac{UH^{-1}\mathring{p}_i}{|H^{-1}\mathring{p}_i|}
\end{array}
\end{equation}
with some $U \in \mathfrak{sl}(3)$. From \eqref{:eq:drho}, one deduces that
\[
\mathfrak{s}_i = \mathrm{ker}(\mathrm{d}\rho_{\mathring{p}_i}(I)) = \{ U \in \mathfrak{sl}(3) \mid \pi_{\mathring{p}_i} U\mathring{p}_i =0\}
\]
with $\pi_x := (I - x x^\top)$ for all $x\in \mathbb{S}^2$. Thus,
\[
\bigcap_{i=1}^n \mathfrak{s}_i = \{ U \in \mathfrak{sl}(3) \mid \pi_{\mathring{p}_i} U\mathring{p}_i =0, \forall i = 1,\cdots,n \}
\]
Now, let us determine $U \in \mathfrak{sl}(3)$ such that $\pi_{\mathring{p}_i} U\mathring{p}_i =0$, for all $i = 1,\cdots,n$. The relation $\pi_{\mathring{p}_i} U\mathring{p}_i =0$ can be equivalently written as
\[
U \mathring{p}_i = \lambda_i  \mathring{p}_i,
\]
with $\lambda_i :=(\mathring{p}_i^\top U  \mathring{p}_i)$. From here, one deduces that $\lambda_i$ are eigenvalues of $U$ and $\mathring p_i$ are the associated eigenvectors. Since the set ${\mathcal M}_n$ of the measured directions $\mathring{p}_i$ is consistent, without loss of generality we assume that the subset ${\mathcal M}_4 = \{\mathring{p}_1,\mathring{p}_2,\mathring{p}_3,\mathring{p}_4\}$ is consistent. Thus, there exist 3 non-null numbers $b_1$, $b_2$, and $b_3$ such that
$\mathring{p}_4 = \sum_{i=1}^3 b_i \mathring{p}_i$. From here,
using the fact that
\[
U \mathring p_4 = U \sum_{i=1}^3 b_i \mathring{p}_i = \sum_{i=1}^3 b_i U \mathring{p}_i= \sum_{i=1}^3 b_i \lambda_i \mathring p_i
\]
and
\[
U \mathring p_4 = \lambda_4 \mathring p_4 = \lambda_4 \sum_{i=1}^3 b_i \mathring p_i  = \sum_{i=1}^3 b_i \lambda_4\mathring p_i
\]
one deduces
\begin{equation}\label{:eq:Udeduction1}
\sum_{i=1}^3 b_i \lambda_i \mathring p_i = \sum_{i=1}^3 b_i \lambda_4 \mathring p_i.
\end{equation}
Since $b_i$ (with $i=1,2,3$) are non-null and the 3 unit vectors $\mathring{p}_i$ (with $i=1,2,3$) are linearly independent, \eqref{:eq:Udeduction1} directly yields $\lambda_1 = \lambda_2 = \lambda_3 = \lambda_4$. Let $\lambda$ denote the value of these four identical eigenvalues.

From here, one easily deduces that the {\em geometric multiplicity} of the eigenvalue $\lambda$ (defined as the dimension of the eigenspace associated with $\lambda$) is equal to 3, since the 3 eigenvectors $\mathring{p}_i$ (with $i=1,2,3$) associated with $\lambda$ are linearly independent. Since the {\em algebraic multiplicity} of the eigenvalue $\lambda$ is no less than the corresponding {\em geometric multiplicity}, one deduces that it is also equal to 3. This means that $U$ has a triple eigenvalue $\lambda$. Since the number of linearly independent eigenvectors of $U$ is equal to 3, the matrix $U$ is diagonalizable. Then, the diagonalizability of $U$ along with the fact that it has a triple eigenvalue implies that $U = \lambda I$. This in turn yields $\mathrm{tr}(U) = 3 \lambda$, which is null since $U$ is an element of $\mathrm{sl}(3)$. Consequently, $\lambda = 0$ and $U= 0$. One then deduces $\bigcap_{i=1}^n \mathfrak{s}_i = \{0\}$ which concludes the proof.
\end{proof}

\vspace{0.2cm}
Now that the non-degenerate right-invariant cost function ${\mathcal{C}}_{\mathring{p}}(\hat H,p)$ is defined, we compute the innovation term $\Delta(\hat H,p)$ as \cite[Eq. (40)]{MahonyNOLCOS13}
\begin{equation}\label{:eq:innoGradGeneral}
\Delta(\hat H,p) := (\mathrm{grad}_1 {\mathcal{C}}_{\mathring{p}}(\hat H,p) )\hat H^{-1},
\end{equation}
where $\mathrm{grad}_1$ is the {\em gradient} in the first variable, using a {\em right-invariant Riemannian metric} on $\mathrm{SL}(3)$. Let $\langle\cdot,\cdot\rangle: \mathfrak{sl}(3)\times \mathfrak{sl}(3) \longrightarrow \mathbb{R}$ be a positive definite inner product on $\mathfrak{sl}(3)$, chosen to be the Euclidean matrix inner product on $\mathbb{R}^{3\times 3}$. Then, a right-invariant Riemannian metric on $\mathrm{SL}(3)$ induced by the inner product $\langle\cdot,\cdot\rangle$ is defined by
\[
\langle U_1H, U_2 H\rangle_H := \langle U_1,U_2\rangle.
\]

\begin{lemma}\label{:lem:innovation}
The innovation term $\Delta(\hat H,p)$ defined by \eqref{:eq:innoGradGeneral} is right equivariant and explicitly given by
\begin{equation}\label{:eq:innoGrad}
\Delta(\hat H,p) =  -\sum_{i=1}^{n} k_i \pi_{e_i}\mathring{p}_ie_i^\top ,\,\,\, \mathrm{with} \,\,\, e_i = \frac{\hat H p_i}{|\hat H p_i|}.
\end{equation}
\end{lemma}
\begin{proof}
The proof for $\Delta(\hat H,p)$ to be equivariant is a direct result of \cite{MahonyNOLCOS13}. Now, using standard rules for transformations of Riemannian gradients and the fact that the Riemannian metric is right-invariant, one obtains
\begin{equation}\label{:eq:GradDerivation1}
\begin{split}
\mathcal{D}_1{\mathcal{C}}_{\mathring{p}}(\hat H,p) [U \hat H]&=\langle\mathrm{grad}_1 {\mathcal{C}}_{\mathring{p}}(\hat H,p), U \hat H \rangle_H \\
&= \langle \mathrm{grad}_1 {\mathcal{C}}_{\mathring{p}}(\hat H,p) \hat H^{-1} \hat H, U\hat H \rangle_H\\
&= \langle \mathrm{grad}_1 {\mathcal{C}}_{\mathring{p}}(\hat H,p) \hat H^{-1} , U \rangle\\
&= \langle \Delta(\hat H,p)  , U \rangle,
\end{split}
\end{equation}
with some $U \in \mathfrak{sl}(3)$. Besides, in view of \eqref{:eq:aggregateCost} one has
\begin{equation}\label{:eq:GradDerivation2}
\begin{array}{ll}
&\!\!\!\!\!\!\mathcal{D}_1{\mathcal{C}}_{\mathring{p}}(\hat H,p) [U \hat H] = d_{\hat H}{\mathcal{C}}_{\mathring{p}}(\hat H,p)[U \hat H] \\[1ex]
&\,\, = \sum_{i=1}^n k_i \left(\frac{\hat H p_i}{|\hat H p_i|} - \mathring p_i\right)^{\!\!\top} \!\!
\left(I - \frac{(\hat H p_i)(\hat H p_i)^\top}{|\hat H p_i|^2}\right) \frac{(U\hat H)p_i}{|\hat H p_i|}\\[1ex]
&\,\, = \sum_{i=1}^n k_i(e_i -\mathring p_i)^\top (I -e_ie_i^\top) U e_i\\[1ex]
&\,\, = \mathrm{tr}\left(\sum_{i=1}^n k_i e_i (e_i -\mathring p_i)^\top (I -e_ie_i^\top) U \right) \\[1ex]
&\,\, = -\mathrm{tr}\left(\sum_{i=1}^n k_i e_i \mathring p_i^\top \pi_{e_i} U \right) \\[1ex]
&\,\, = \left \langle  -\sum_{i=1}^n k_i\pi_{e_i}\mathring p_i e_i^\top, U \right\rangle .
\end{array}
\end{equation}
Finally, the explicit expression of $\Delta(\hat H,p)$ given by \eqref{:eq:innoGrad} is directly obtained from \eqref{:eq:GradDerivation1} and  \eqref{:eq:GradDerivation2}.
\end{proof}

\vspace{0.2cm}
One deduces from \eqref{:eq:innoGrad} that
\[
\Delta(E,\mathring p) =  -\sum_{i=1}^{n} k_i \pi_{e_i}\mathring{p}_ie_i^\top = \Delta(\hat H,p),\,\,\, \mathrm{with} \,\,\, e_i = \frac{E \mathring p_i}{|E \mathring p_i|},
\]
and, consequently, from \eqref{:eq:ErrorObserverGeneral} that
\begin{equation}\label{:eq:ErrorObserverDyn}
\dot E = \left(\sum_{i=1}^{n} k_i \pi_{e_i}\mathring{p}_ie_i^\top \right) E = - \mathrm{grad}_1  {\mathcal{C}}_{\mathring{p}}(E,\mathring p).
\end{equation}

\begin{theorem}\label{theo:1}
Consider the kinematics \eqref{:eq:dotH} and assume that the velocity group $U \in \mathfrak{sl}(3)$ is known. Consider the nonlinear observer defined by \eqref{:eq:ObserverGeneral}, with the innovation term $\Delta(\hat H,p) \in \mathfrak{sl}(3)$ defined by \eqref{:eq:innoGrad}. Then, if the set ${\cal M}_n$ of the measured directions $\mathring{p}_i$ is consistent, there exists a basin of attraction ${\mathcal B} \subseteq \mathrm{SL}(3)$ containing $I$ such that for any initial conditions $H(0)$ and $\hat H(0) \in \mathrm{SL}(3)$ such that $E(0) \in \mathcal{B}$ then $E(t) \rightarrow I$ and $\hat H(t) \rightarrow H(t)$.
\end{theorem}
\begin{proof}
This theorem can be seen as a direct result of Theorem 2 in \cite{MahonyNOLCOS13}, but it can also be proved using classical Lyapunov theory. The candidate Lyapunov function under consideration is ${\mathcal L}_0 := {\mathcal{C}}_{\mathring{p}}(E, \mathring p)$. Using the consistency of the set ${\cal M}_n$, one can ensure that ${\mathcal L}_0$ is locally a positive definite function of $E$. The time-derivative of ${\mathcal L}_0$ along the error flow \eqref{:eq:ErrorObserverDyn} verifies
\[
\begin{array}{lcl}
\dot {\mathcal L}_0 &=& \frac{\mathrm{d}}{\mathrm{d}t} \sum_{i=1}^n \frac{k_i}{2} \left| \frac{E \mathring p_i}{|E \mathring p_i|} - \mathring p_i\right|^2 \\
&=& \sum_{i=1}^n k_i ( \frac{E \mathring p_i}{|E \mathring p_i|} - \mathring p_i)^\top(I - \frac{(E \mathring p_i)(E \mathring p_i)^\top}{|E \mathring p_i|^2}) \frac{\dot E \mathring p_i}{|E \mathring p_i|} \\
&=& -\sum_{i=1}^n k_i(e_i - \mathring p_i)^\top \pi_{e_i} \Delta(E,\mathring p) e_i \\
&=& \mathrm{tr}(\sum_{i=1}^n k_i e_i \mathring p_i^\top \pi_{e_i} \Delta(E,\mathring p))\\
&=& - ||\Delta(E,\mathring p)||^2 ,
\end{array}
\]
with $||\cdot||$ denoting the Frobenius norm defined by $||A||= \sqrt{\langle A, A \rangle}$ for any real valued square matrix $A$.
From here, one ensures that $E$ is locally bounded. Moreover, by application of LaSalle's theorem, one deduces that $\Delta(E,\mathring p)$ converges to zero.
From the definitions of $\Delta(\cdot,\cdot)$ \eqref{:eq:innoGrad} and $e_i$ \eqref{:eq:outputErrors}, one deduces that
\[
\Delta(E,\mathring p) {E}^{-\top} = \sum_{i=1}^n \left(I - \frac{E \mathring p_i {\mathring p_i}^\top E^\top}{|E {\mathring p_i}|^2}\right)\frac{{\mathring p_i} {\mathring p_i}^\top}{|E {\mathring p_i}|}.
\]
Computing the trace of $\Delta(E,\mathring p) {E}^{-\top}$, it follows:
\[
\text{tr}(\Delta(E,\mathring p) {E}^{-\top}) = \sum_{i=1}^n \frac{1}{|E {\mathring p_i}|^3}\left(|E {\mathring p_i}|^2 |{\mathring p_i}|^2 - ((E {\mathring p_i})^\top{\mathring p_i})^2\right).
\]
Define $X_i=E {\mathring p_i}$ and $Y_i=\mathring p_i$, and it is straightforward to verify that
\[
\text{tr}(\Delta(E,\mathring p) E^{-\top}) =  \sum_{i=1}^n \frac{1}{|X_i|^3}\left(|X_i|^2 |Y_i|^2 - (X_i^\top Y_i)^2\right) \geq 0
\]
Using the fact that $\Delta=0$ at the equilibrium and therefore $\text{tr}(\Delta(E,\mathring p) {E}^{-\top})=0$, as well as the
Cauchy-Schwarz inequality, it follows that $X^\top_i Y_i = \pm |X_i||Y_i|$ and consequently one has:
\[
(E {\mathring p_i})^\top{\mathring p_i} = \pm |E {\mathring p_i}||\mathring p_i|,\; \forall i=\{1, \cdots, n\},
\]
which in turn implies the existence of some scalars $\lambda_i= \pm |E {\mathring p_i}|$ such that
\begin{equation} \label{Htpi}
E {\mathring p_i} = \lambda_i {\mathring p_i}.
\end{equation}
Using \eqref{Htpi} and exploiting the consistency of the set ${\cal M}_n$, one can proceed analogously to the proof of Lemma \ref{:lem:aggregateCost} to deduce that $E$ has a triple eigenvalue $\lambda$ and $E = \lambda I$. Then, evoking the fact that $\det(E)=1$, one deduces that $\lambda =1$ and $E=I$. Consequently, $E$ converges asymptotically to the identity $I$.
\end{proof}

\begin{remark}\label{rem:domainstab}
The boundaries of the stability domain associated with Theorem~\ref{theo:1} are extremely difficult, and probably impossible, to analytically characterise.
The nature of the Lyapunov function ${\mathcal{L}}_0$ is always well conditioned in the neighbourhood of the correct homography matrix, but the global geometry of $\mathrm{SL}(3)$ is complex and there will be critical points and non-convex cost structure when the rotation component of the homography matrix has errors approaching $\pi\,$rads from the true homography. The authors believe, based on extensive simulation studies (cf.~\S\ref{section:experiment}) and our intuition for such problems, that the stability domain is very large in practice, encompassing all realistic scenarios where the camera actually observes the desired scene (rotations of less than $\pi/2\,$rads and moderate displacements).
\end{remark}

\section{Application to robotic systems} \label{sec:application}
\subsection{Homography kinematics from a camera moving with rigid-body motion}
\label{sec:BFF_velocity}
In this section, we consider the case where a sequence of homographies is generated by a moving camera viewing a stationary planar surface. The goal is to
develop a nonlinear filter for the image homography sequence using the velocity associated with the rigid-body motion of the camera rather than the group
velocity of the homography sequence, as was assumed in Section \ref{sec:Observer}. In fact, any group velocity (infinitesimal variation of the homography) must be associated with an instantaneous variation in measurement of the \emph{current} image ${\cal A}$ and not with a variation in the \emph{reference} image ${\cal \mathring{A}}$.  This imposes constraints on two degrees of freedom in the homography velocity, namely those associated with
variation of the normal to the reference image, and leaves the remaining six degrees of freedom in the homography group velocity depending on the rigid-body velocities of the camera.

Denote the rigid-body angular velocity and linear velocity of $\{A\}$ with respect to $\{\mathring{A}\}$ expressed in $\{A\}$ by $\Omega$ and $V$, respectively. The rigid body kinematics of $(R, \xi)$ are given by:
\begin{align}
  \dot{R} & = R \Omega_\times \label{eq:RBkin_R}\\
  \dot{\xi} & = R V \label{eq:RBkin_xi}
\end{align}
where $\Omega_\times$ is the skew symmetric matrix associated with the vector cross-product, i.e. $\Omega_\times y =\Omega \times y$, for all $y \in \mathbb{R}^3$.

Recalling \eqref{eq:dB}, it is easily verified that:
\[
\dot{d} = -\eta^\top V, \quad\quad \frac{d}{dt} \mathring{d}= 0.
\]

This constraint on the variation of $\eta$ and $\mathring{d}$ is precisely the velocity constraint associated with the fact that the reference image is stationary.

Consider a camera attached to the moving frame ${\cal A}$ moving with kinematics \eqref{eq:RBkin_R} and \eqref{eq:RBkin_xi} viewing a stationary planar scene.   The group velocity $U \in \gothsl(3)$ induced by the rigid-body motion, and such that the dynamics of $H$ satisfies \eqref{:eq:dotH},
is given by \cite[Lem. 5.3]{2012_Mahony.IJC}
\begin{equation}\label{eq:dotH_final}
U = \left( \Omega_\times +\frac{V \eta^\top }{d} -\frac{\eta^\top V }{3 d}I  \right).
\end{equation}

Note that the group velocity $U$ induced by camera motion depends on the additional variables $\eta$ and $d$ that define the scene geometry at time $t$ as well as the scale factor $\gamma$. Since these variables are unmeasurable and cannot be extracted directly from the measurements, in the sequel, we rewrite:
\begin{equation}
U:= \left(\Omega_\times +\Gamma\right), \quad \mbox{ with } \Gamma=\left (\frac{V \eta^\top }{d} -\frac{\eta^\top V }{3 d}I  \right).
\label{eq:A}
\end{equation}

Since $\{\mathring{A}\}$ is stationary by assumption, the vector $\Omega$ can be directly obtained from the set of embedded gyroscopes. The term $\Gamma$ is related to the translational motion expressed in the current frame $\{A\}$. If we assume that $\frac{\dot{\xi}}{d}$ is constant ({\em e.g.} the situation in which the camera moves with a constant velocity parallel to the scene or converges exponentially toward it), and using the fact that $V=R^\top\dot{\xi}$, it is straightforward to verify that
\begin{equation}
\dot{\Gamma}=[\Gamma, \Omega_\times], \quad \mbox{ with } \Gamma=\left (\frac{V \eta^\top }{d} -\frac{\eta^\top V }{3 d}I  \right),
\label{Gamma}
\end{equation}
where $[\Gamma, \Omega_\times]= \Gamma\Omega_\times -\Omega_\times \Gamma$ is the Lie bracket.

However, if we assume that $\frac{V}{d}$ is constant (the situation in which the camera follows a circular trajectory over the scene or performs an exponential convergence towards it), it follows that
\begin{equation}
\dot{\Gamma}_1=\Gamma_1 \Omega_\times, \mbox{ with } \Gamma_1=\frac{V}{d}\eta^\top.
\label{Gamma_1}
\end{equation}

\subsection{Observer with partially known velocity of the rigid body}
In this section we assume that the group velocity $U$ in \eqref{eq:A} is not available, but the angular velocity $\Omega$ is. The goal is to provide an estimate $\hat{H} \in \mathrm{SL(3)}$ for $H \in \mathrm{SL(3)}$ to drive the group error $E\,(=\hat{H}H^{-1})$ to the identity matrix $I$ and the error term $\tilde{\Gamma}=\Gamma-\hat{\Gamma}$ (resp. $\tilde{\Gamma}_1=\Gamma_1-\hat{\Gamma}_1$) to $0$ if $\Gamma$ (resp. $\Gamma_1$) is constant or slowly time varying. The observer when $\frac{\dot{\xi}}{d}$ is constant is chosen as follows (compare to \eqref{:eq:ObserverGeneral}):
\begin{align}
\dot{\hat{H}} &= \hat{H}(\Omega_\times+\hat{\Gamma})-\Delta(\hat H,p)\hat{H}, \label{:eq:dot_hat_H}\\
\dot{\hat{\Gamma}} &= [\hat{\Gamma}, \Omega_\times] -k_I\Ad_{\hat{H}^{\top}} \Delta(\hat H,p)  .\label{:eq:dot_hat_Gamma}
\end{align}
and the observer when $\frac{V}{d}$ is constant is defined as follows:
\begin{align}
\dot{\hat{H}} &= \hat{H}(\Omega_\times+\hat{\Gamma}_1-\frac{1}{3} \tr(\hat{\Gamma}_1)I)-\Delta(\hat H,p)\hat{H}, \label{:eq:dot_hat_H-2}\\
\dot{\hat{\Gamma}}_1 &= \hat{\Gamma}_1\Omega_\times -k_I\Ad_{\hat{H}^{\top}}\Delta(\hat H,p) .\label{:eq:dot_hat_Gamma_1}
\end{align}
with some positive gain $k_I$ and $\Delta(\hat H,p)$ given by \eqref{:eq:innoGrad}.

\begin{proposition}\label{propo:1}
Consider a camera moving with kinematics \eqref{eq:RBkin_R} and \eqref{eq:RBkin_xi} viewing a planar scene. Assume that ${\cal \mathring{A}}$ is stationary and that the orientation velocity $\Omega \in \{A\}$ is measured and bounded. Let $H : {\cal A} \rightarrow {\cal \mathring{A}}$ denote the calibrated homography \eqref{eq:homog-decomp} and consider the kinematics \eqref{:eq:dotH} along with \eqref{eq:A}. Assume that $H$ is bounded and that $\Gamma$ (resp. $\Gamma_1$) is such that it obeys \eqref{Gamma} (resp. \eqref{Gamma_1}).

Consider the nonlinear observer defined by (\ref{:eq:dot_hat_H}--\ref{:eq:dot_hat_Gamma}), (resp. (\ref{:eq:dot_hat_H-2}--\ref{:eq:dot_hat_Gamma_1})) along with the innovation $\Delta(\hat H, p) \in \gothsl(3)$ given by \eqref{:eq:innoGrad}. Then, if the set ${\cal M}_n$ of the measured directions $\mathring{p}_i$ is consistent, the equilibrium $(E, \tilde{\Gamma})=(I, 0)$ (resp. $(E, \tilde{\Gamma}_1)=(I, 0)$) is locally asymptotically stable.
\end{proposition}

\begin{proof}[Proof sketch]
We will consider only the situation where the estimate of $\Gamma$ is used. The same arguments can also be used for the case where the estimate of $\Gamma_1$ is considered. Differentiating $e_i$ \eqref{:eq:outputErrors} and using \eqref{:eq:dot_hat_H} yields
\[\dot{e}_i=-\pi_{e_i}(\Delta+\Ad_{\hat{H}}\tilde \Gamma) e_{i} .\]
Define the following candidate Lyapunov function:
\begin{equation}
\begin{array}{lcl}
{\mathcal L}&=& {\mathcal L}_0 +\frac{1}{2k_I}||\tilde{\Gamma}\|^2 \\
&=&\sum_{i=1}^{n}\frac{k_i}{2}\left |e_i-\mathring{p}_i \right |^2 +\frac{1}{2k_I}||\tilde{\Gamma}\|^2.
\end{array}
\label{lyap-1}\end{equation}
Differentiating  ${\cal L}$ and using the fact that $\tr \left (\tilde{\Gamma}^\top \left (\left [\tilde{\Gamma},\Omega \right ]\right) \right )=0$, it follows:
\begin{align*}
\dot{{\cal L}}=&\sum_{i=1}^{n}k_i(e_i-\mathring{p}_i)^\top\dot{e}_i +\tr \left (\tilde{\Gamma}^\top\Ad_{\hat{H}^{T}}\Delta\right ).
\end{align*}
Introducing the above expression of $\dot{e}_i$ and using the fact that $\tr(AB)=\tr(B^\top A^\top)$, it follows:
\begin{align*}
\dot{{\cal L}}
=& -\!\sum_{i=1}^{n}k_i(e_i-\mathring{p}_i)^\top \pi_{e_i}(\Delta+\Ad_{\hat{H}}\tilde{\Gamma})e_{i} + \tr \!\left (\Ad_{\hat{H}^{-1}}\Delta^\top\tilde{\Gamma}\right )\\
=& \sum_{i=1}^{n}k_i\mathring{p}_i^\top \pi_{e_i}(\Delta +\Ad_{\hat{H}} \tilde{\Gamma})e_{i} + \tr\left(\Ad_{\hat{H}^{-1}}\Delta^\top \tilde{\Gamma}\right ) \\
=& \tr \!\left( \sum_{i=1}^{n}k_ie_{i}\mathring{p}_i^\top\pi_{e_i}(\Delta  +\Ad_{\hat{H}} \tilde{\Gamma}) + \Ad_{\hat{H}^{-1}}\Delta^\top \tilde{\Gamma}\right ) \\
=& \tr \!\left( \sum_{i=1}^{n} k_ie_{i}\mathring{p}_i^\top \pi_{e_i}\Delta
+ \Ad_{\hat{H}^{-1}}(\Delta^{\!\top} + \sum_{i=1}^{n}k_i e_{i}\mathring{p}_i^\top\pi_{e_i}) \tilde{\Gamma}\! \right )
\end{align*}
Finally, introducing the expression of $\Delta$ \eqref{:eq:innoGrad}, one gets
\[\dot{{\cal L}}=- \| \Delta  \|^2.  \]
The derivative of the Lyapunov function is negative semi-definite, and equal to zero when $\Delta=0$. Given that $\Omega$ is bounded, it is easily verified that $\dot{\mathcal{L}}$ is uniformly continuous and Barbalat's Lemma can be used to prove asymptotic convergence of $\Delta \rightarrow 0$. Using the same arguments as in the proof of Theorem \ref{theo:1}, it is straightforward to verify that $E \rightarrow I$. Consequently the left-hand side of the Lyapunov expression \eqref{lyap-1} converges to zero and $\|\tilde{\Gamma}\|^2$ converges to a constant.

Computing the time derivative of $E$ and using the fact that $\Delta$ converges to zero and $E$ converges to $I$, it is straightforward to show that
\[\lim_{t\rightarrow \infty} \dot{E}=-\Ad_{\hat{H}}\tilde{\Gamma}=0.\]
Using boundedness of $H$, one ensures that $\lim_{t\rightarrow \infty}\tilde{\Gamma}=0$.
\end{proof}

\section{Experimental results, Part I}\label{section:experiment}
The proposed homography observers have been experimentally evaluated using an AVT Stingray 125B camera and an Xsens MTi-G IMU attached together. This sensor system, depicted in Fig.~\ref{fig:SensorRefImage}, is also the one used in \cite{ScandaroliIROS2011}. The IMU runs at $200$ Hz and only the gyroscope measurements are used for feeding the observer. The camera provides 40 frames per second at a resolution of $800\times 600$ pixels. It makes use of a Theia SY125M wide-angle lens whose estimated intrinsic parameters are $(448.85, 450.26)$ pixels for the focal length and $(394.30, 292.82)$ pixels for the principle point.

\begin{figure}
\centering
\includegraphics[width= 3.6cm]{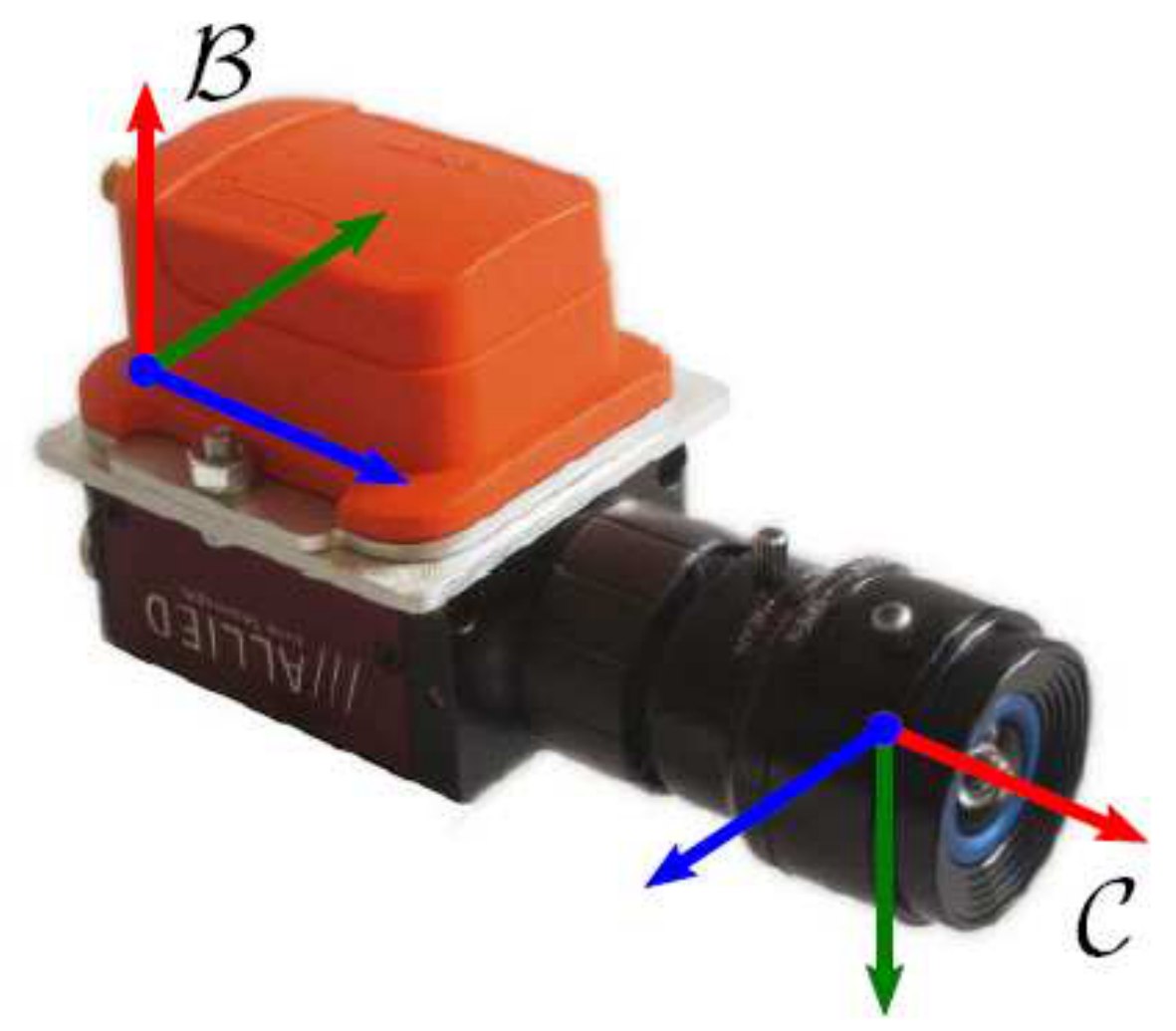} \hspace{0.2cm}
\includegraphics[width= 4cm]{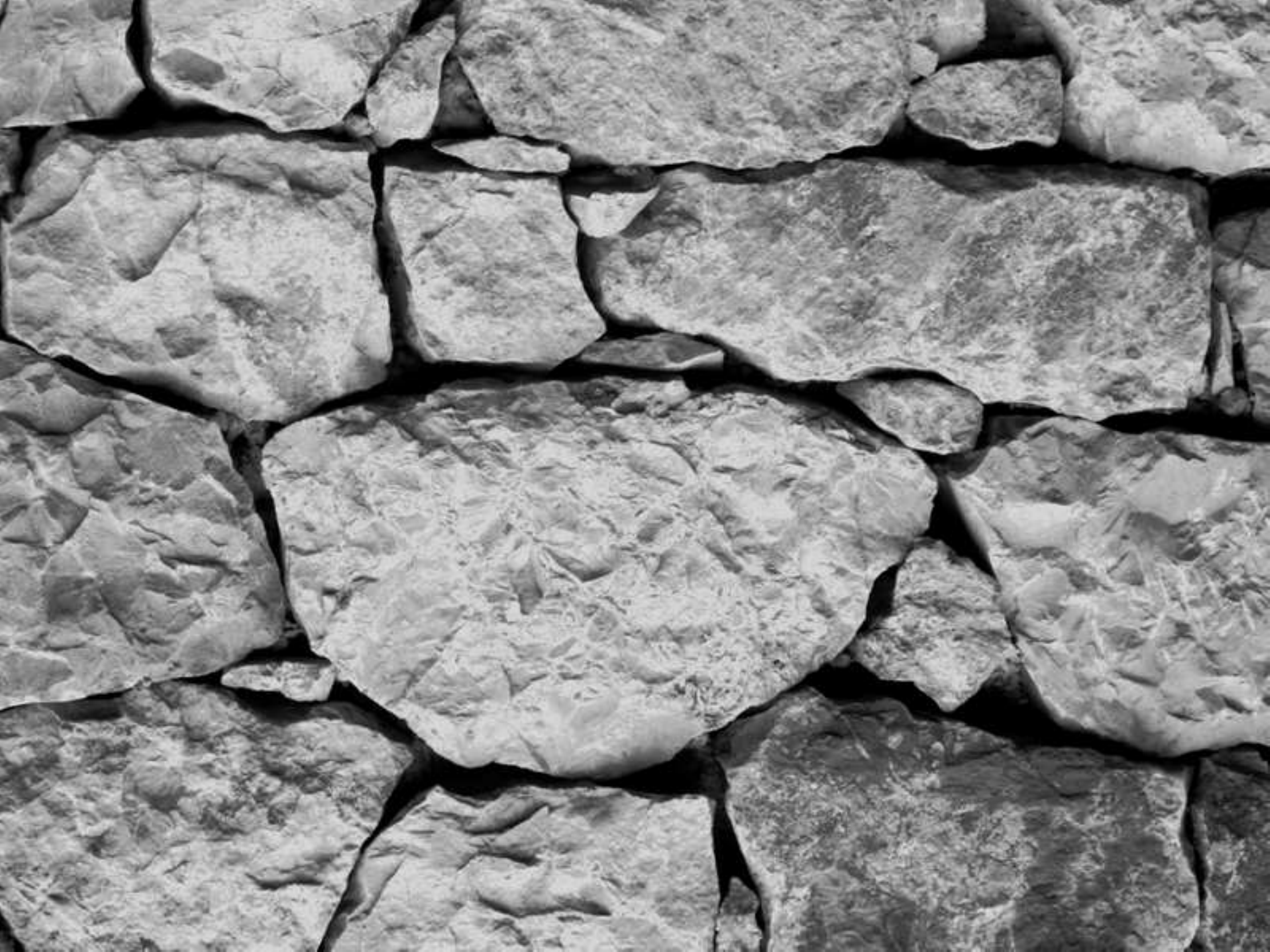}
\caption{Inertial-visual sensor (left) and reference image (right).}\label{fig:SensorRefImage}
 \centering
 \includegraphics[width= 9cm, height = 4cm]{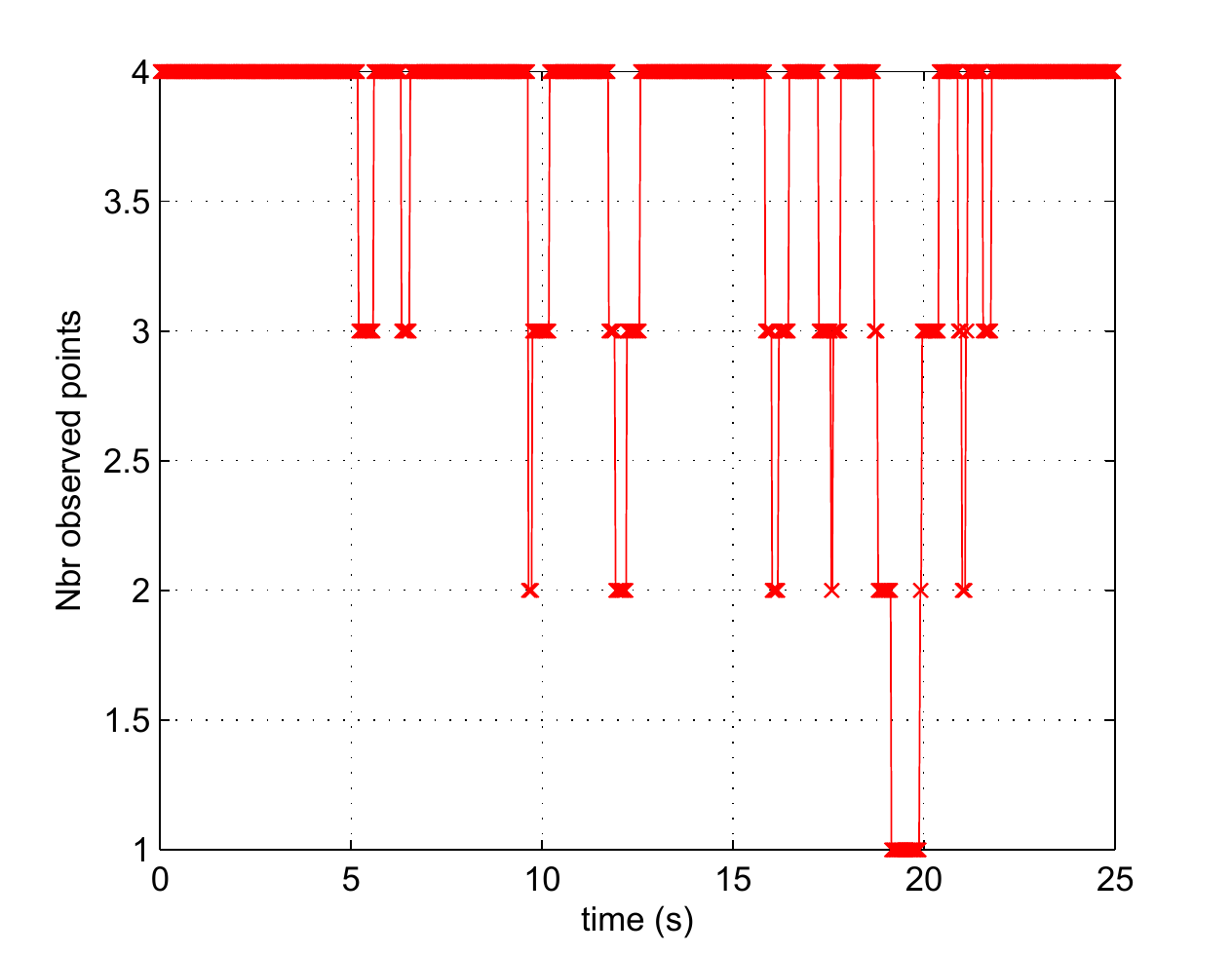}
 \caption{Number of extracted point-features vs. time.}\label{fig:nbrObsPoints}
 \centering
 \includegraphics[width= 9cm]{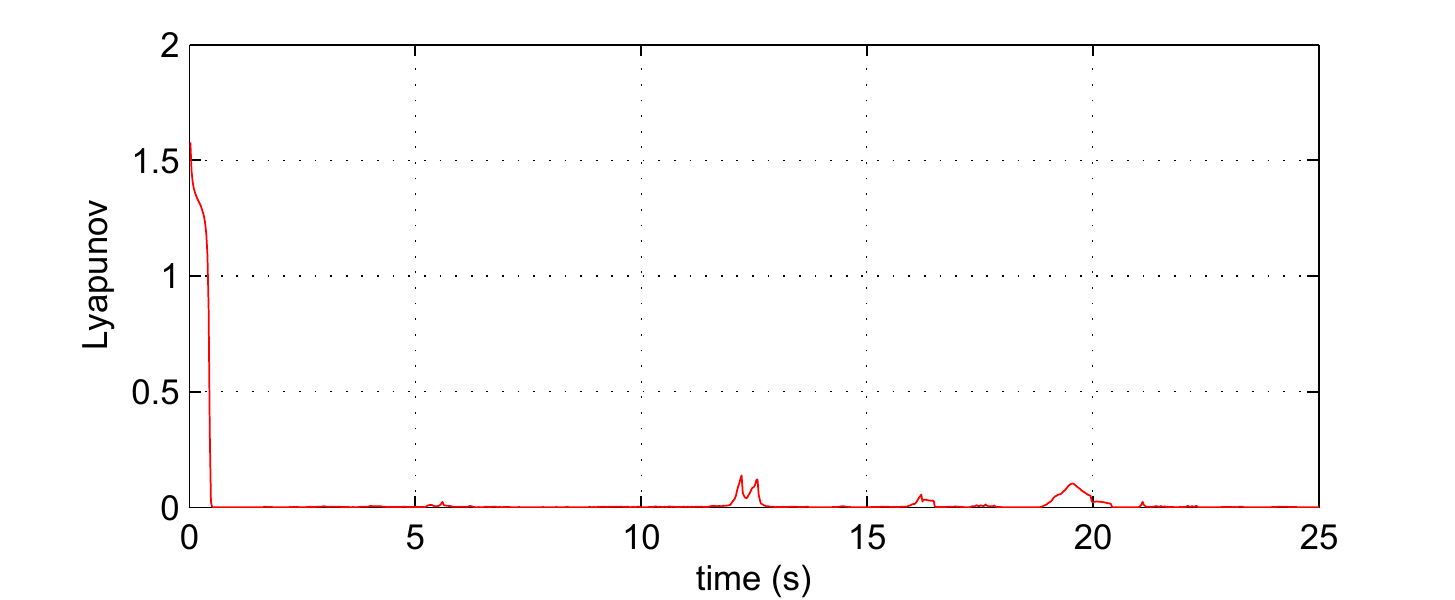}
 \caption{Cost function ${\cal L}_0 = \sum_{i=1}^4 \frac{1}{2}|e_i - \mathring{p}_i|^2$ vs. time. The ``ground-truth'' point locations $\mathring{p}_i$ are simulated using the ESM--homography.}\label{fig:costFunction} \vspace{-0.2cm}
\end{figure}

The reference target for the visual system is the image shown in Fig.~\ref{fig:SensorRefImage} printed on a $376 \times 282$ $\text{mm}^2$ sheet of paper and then placed over a planar surface parallel to the ground. For simplicity, the point-features extracted from the images are only the 4 corners of the target whenever they are visible in the field of view of the camera. This is the minimum number of points necessary for algebraically reconstructing the homography matrix \cite{Hartley2003}. We have chosen this simplistic scenario as a first experiment to be able to demonstrate the performance of our proposed observer algorithms independent of potential problems with feature detection, matching and tracking in the real image data. The latter is not the focus of this paper. The next section describes a more realistic experiment where failures of the image processing stage occur.

There are a plethora of methods for feature detection and feature matching reported in the computer vision literature, the quality of which directly impacts on the precision of the estimated homography. Since this issue is not the main preoccupation of the present study, we have opted to use an Efficient Second-Order Minimization (ESM) dense approach \cite{Benhimane-IJRR-07} to directly estimate the homography matrix for simulated ``ground-truth'' comparison and also for computing the image coordinates of the 4 corners of the target.

\begin{figure*}
\includegraphics[trim=1cm 9.5cm 0cm 10cm,width=19cm]{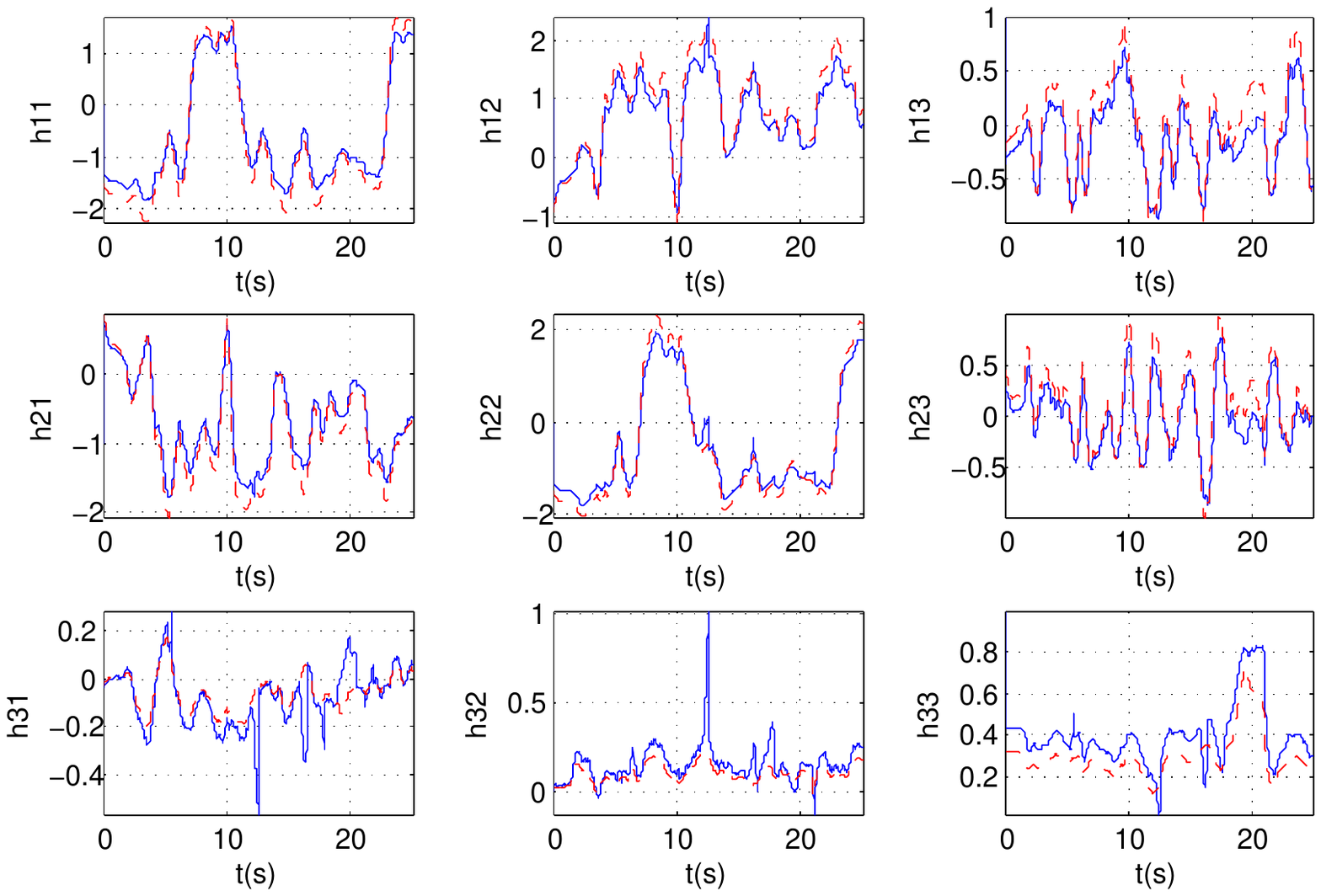}
\caption{Homography estimated by the proposed observer (solid line) and simulated ``ground-truth'' ESM--homography (dashed red line) vs. time.} \label{fig:estimatedHomography}
\end{figure*}

The experimental results reported next correspond to the observer \eqref{:eq:dot_hat_H}--\eqref{:eq:dot_hat_Gamma} whose gains are chosen as $k_P = 60$ and $k_I =1$. Let us emphasize that measurements of linear velocity are not available for the observer in this experiment. We have performed a hand-held movement of the camera such that the target's corners frequently exit the camera's field of view. Fig.~\ref{fig:nbrObsPoints} shows that the number of extracted point-features drops below $4$ at multiple instances and for extended periods of time. These critical cases allow us to test the robustness of the proposed observer. Fig.~\ref{fig:costFunction} shows the time evolution of the cost function ${\cal L}_0 = \sum_{i=1}^4 \frac{1}{2}|e_i - \mathring{p}_i|^2$ related to the homography error, cf. \eqref{lyap-1}. The ``ground-truth'' point locations $\mathring{p}_i$ are simulated using the ESM--homography. Fig.~\ref{fig:estimatedHomography} illustrates the time evolution of the homography estimated by the proposed observer and of the simulated ``ground-truth'' homography obtained by the ESM algorithm. The first video in the supplementary multimedia material presents the image sequence with highlighted red point-features, a green quadrilateral joining the estimates of the 4 corners of the target, a warped current image, together with the time evolution of the cost function ${\cal L}_0$. The video is also available at \cite{weblinkvideo1}.

Observe from Figs. \ref{fig:costFunction} and \ref{fig:estimatedHomography} that during the initial transient (before $1$s) the cost function ${\cal L}_0$ converges quickly and then remains close to zero for the following eleven seconds. The estimated homography also converges quickly and remains close to the simulated ``ground-truth'' ESM--homography. The occasional loss of point-features during that period only marginally affects the overall performance of the proposed observer. After that first period, there are three periods where the precision of the estimated homography is significantly deteriorated due to the loss of extracted point-features: from $12$s to $13$s, from $16$s to $16.5$s, and particularly from $19$s to $20.5$s. Whenever the number of extracted point-features is not sufficient ($<4$), the estimated homography is updated mainly based on the forward integration of the kinematic equation using the angular velocity measurements. Since the linear velocity is not measured in this experiment, strong divergence of the homography estimation error is almost inevitable in this situation. This is particularly the case for fast translational camera motion. However, once the set of extracted point-features becomes consistent again, the estimated homography quickly converges to the simulated ``ground-truth'' homography, demonstrating the robustness of the proposed observer algorithm.

\section{Experimental results, part II}\label{sec:experimentII}
In this section, we present an application of our approach to image stabilization in the presence of very fast camera motion, severe occlusion and specular reflections. The experiment consists of moving a camera with attached IMU rapidly in front of a planar scene as shown in Fig.~\ref{fig:e1}, top left. The aim is to track and stabilize a region of interest in the reference image, shown in Fig.~\ref{fig:e1}, bottom left. The complete result is shown in the second video provided as supplementary material and available at \cite{weblinkvideo2}, which demonstrates highly robust performance under very challenging conditions.

\begin{figure*}[ht!]
\includegraphics[width=18cm]{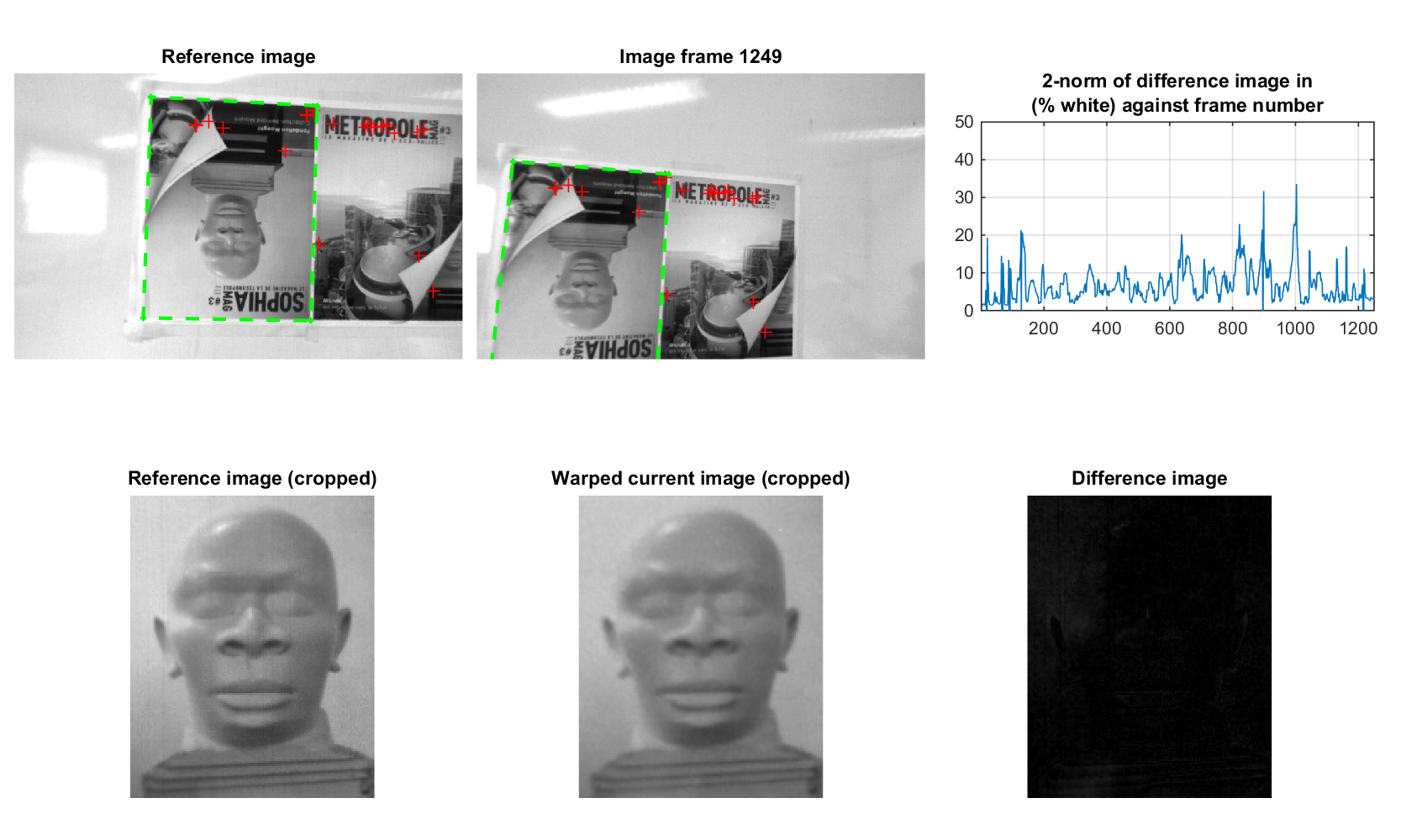}
\caption{Final result for image stabilization experiment, tracking a region of interest through 900 frames at 20 frames per second.
Bottom row from left to right: region of interest in the reference image; region of interest in the current video frame computed using the estimated homography; difference image.
Top row from left to right: reference image with point-features marked; current frame with corresponding point-features and estimated homography (represented by the green dashed rectangle); 2-norm of the reference image (100\% = all white) against frame number.}
\label{fig:e1} \vspace{-0.4cm}
\end{figure*}

The camera used is a prototype synchronized camera-IMU combination with an Aptina MT9V034 CMOS sensor and an Analog Devices  ADIS16375 MEMS IMU. The IMU runs at $200$ Hz, providing an angular velocity measurement to the observer. The camera provides $20$ frames per second at a resolution of $752\times 480$ pixels. The estimated camera parameters are $(464.66, 462.75)$ pixels for the focal length and $(385.63,227.53)$ pixels for the principle point.

Point-features are extracted using the {\tt detectSURFFeatures} routine in MATLAB R2012b with the standard recommended parameter settings and matched using MATLAB's KLT-based \cite{lucas,tomasi} {\tt matchFeatures} routine with a matching threshold of $0.05$. We have purposefully avoided using more sophisticated image processing routines in order to demonstrate the raw performance of our observer. It is likely that the results could be further improved by the use of higher quality feature matching algorithms. It is quite unrealistic to track one and the same set of features through a long video sequence, in particular given the low frame rate and comparatively rapid motion in our test sequence as well as the presence of severe occlusion (particularly strong in frames $612-617$ and frames $713-721$) and large specular reflections (e.g. frames $344-363$). We have hence opted to match features between the reference image and each subsequent frame separately. To do this, we first forward integrate the observer equations \eqref{:eq:dot_hat_H}--\eqref{:eq:dot_hat_Gamma} using only the gyroscope measurements, i.e. setting the observer gains $k_i$ ($i=1,\cdots,n$) and $k_I$ to zero. We then use the resulting homography estimate to transform the current image before we apply feature extraction and matching. A KLT-based matching algorithm is well suited to this approach since it favors translational motion over rotational motion, and most of the rotational motion has already been compensated for by forward integrating the angular velocity. We remove remaining $n$-to-$1$ matches (an artefact of MATLAB's KLT implementation) and apply a conservative symmetric bandpass median filter (60\% passband) for outlier removal. Again, we have purposefully avoided the use of more sophisticated (and much more computationally expensive) alternatives such as RANSAC \cite{ransac}. Finally, we use observer gains of $k_i=60$ ($i=1,\cdots,n$) and $k_I=0.0375$ to rapidly iterate the observer equations $1000$ times per video frame. The computational effort for this last step is negligible compared to the previous image processing steps.

Note that the procedure described above does not satisfy the assumptions of our theoretical convergence results, not least because different point-features are selected for every video frame. However, it is intuitively clear that such a procedure will still work as long as the selected feature set provides persistent excitation to the observer. Indeed, our experimental results strongly support this claim.

The experimental results (cf. the second video in the supplementary material and available at \cite{weblinkvideo2}) show good and robust performance throughout the entire video sequence, including the previously mentioned passages with severe occlusion and large specular reflections. Even when the number of usable point-feature matches drops below four for a number of frames (e.g. frames $533-541$), or when our algorithm selects a wrong feature match (e.g. frame $719$), the observer continues to track the region of interest well and quickly recovers from any tracking errors.

\section{Conclusions}
In this paper we developed a nonlinear observer for a sequence of homographies represented as elements of the Special Linear group $\mathrm{SL}(3)$. The observer directly uses point-feature correspondences from an image sequence without requiring explicit computation of the individual homographies between any two given images and fuses these measurements with measurements of angular velocity from onboard gyroscopes using the correct Lie group geometry. The stability of the observer has been proved for both cases of known full group velocity and known rigid-body velocities only. Even if the characterization of the stability domain still remains an open issue, simulation and experimental results have been provided as a complement to the theoretical approach to demonstrate a large domain of stability. A potential application to image stabilization in the presence of very fast camera motion, severe occlusion and specular reflections has been demonstrated with very encouraging results even for a relatively low video frame rate.

\section*{Acknowledgments}
This work was supported by the {\em French Agence Nationale de la Recherche} through the ANR ASTRID SCAR project ``Sensory Control of Aerial Robots'' (ANR-12-ASTR-0033), and the {\em Chaire d'excellence en robotique RTE-UPMC}, and
the {\em Australian Research Council} through the ARC Discovery Project DP0987411 ``State Observers for Control Systems with Symmetry''. Special thanks go to S. Omari for helping us with the data acquisition for experimental validation.


\begin{thebibliography}{10}
\bibitem{weblinkvideo1}
\newblock Experiment video - part I.
\newblock URL: \url{http://youtu.be/KDSonH9Oe34}.

\bibitem{weblinkvideo2}
\newblock Experiment video - part II.
\newblock URL: \url{https://youtu.be/GbC7QbnwqqM}.

\bibitem{Benhimane-IJRR-07}
S. Benhimane and E. Malis.
\newblock Homography-based 2D visual tracking and servoing.
\newblock {\em International Journal of Robotics Research}, 26(7): 661--676, 2007.

\bibitem{Bonnabel}
S. Bonnabel, P. Martin, and P. Rouchon.
\newblock Non-linear symmetry-preserving observers on Lie groups.
\newblock {\em IEEE Transactions on Automatic Control}, 54(7): 1709--1713, 2009.

\bibitem{2007_Caballero_icra}
F. Caballero, L. Merino, J. Ferruz, and A. Ollero.
\newblock Homography based Kalman filter for mosaic building. Applications to UAV position estimation.
\newblock In {\em Proceedings of the IEEE International Conference on Robotics and Automation (ICRA)}, pages 2004--2009, 2007.

\bibitem{2011_de_Plinval_icra}
H. de Plinval, P. Morin, P. Mouyon, and T. Hamel.
\newblock Visual servoing for underactuated VTOL UAVs: a linear homography-based approach.
\newblock In {\em Proceedings of the IEEE International Conference on Robotics and Automation (ICRA)}, pages 3004-3010, 2011.

\bibitem{1998_Dellaert_iros}
F. Dellaert, C. Thorpe, and S. Thrun.
\newblock Super-resolved texture tracking of planar surface patches.
\newblock In {\em Proceedings of the IEEE/RSJ International Conference on Intelligent Robots and Systems (IROS)}, pages 197--203, 1998.

\bibitem{2005_Fang_TSMC}
Y. Fang, W. Dixon, D. Dawson, and P. Chawda.
\newblock Homography-based visual servoing of wheeled mobile robots.
\newblock {\em IEEE Transactions on Systems, Man, and Cybernetics - Part B}, 35: 1041--1050, 2005.

\bibitem{FaugerasLustman}
O. Faugeras and F. Lustman.
\newblock Motion and structure from motion in a piecewise planar environment.
\newblock {\em International Journal of Pattern Recongnition and Artificial Intelligence}, 2(3): 485--508, 1988.

\bibitem{ransac}
M. A. Fischler and R. C. Bolles.
\newblock Random sample consensus: A paradigm for model fitting with applications to image analysis and automated cartography.
\newblock Communications of the ACM, 24(6): 381-–395, 1981.

\bibitem{2007_Fraundorfer_iros}
F. Fraundorfer, C. Engels, and D. Nister.
\newblock Topological mapping, localization and navigation using image collections.
\newblock In {\em Proceedings of the IEEE/RSJ International Conference on Intelligent Robots and Systems (IROS)}, pages 3872--3877, 2007.

\bibitem{HuaCDC2011}
T. Hamel, R. Mahony, J. Trumpf, P. Morin, and M.-D. Hua.
\newblock Homography estimation on the Special Linear group based on direct point correspondence.
\newblock In {\em Proceedings of the 50th IEEE Conference on Decision and Control (CDC)}, pages 7902--7908, 2011.

\bibitem{Hartley2003}
R. Hartley and A. Zisserman.
\newblock {\em Multiple View Geomerty in Computer Vision}, Cambridge University Press, second edition, 2003.

\bibitem{2014_Hua_ifac}
M.-D. Hua, G. Allibert, S. Krupoinski, and T. Hamel.
\newblock Homography-based Visual Servoing for Autonomous Underwater Vehicles.
\newblock In {\em Proceedings of the 19th IFAC World Congress}, pages 5726--5733, 2014.

\bibitem{toappear_Lageman.TAC}
C. Lageman, J. Trumpf, and R. Mahony.
\newblock Gradient-like observers for invariant dynamics on a Lie group.
\newblock {\em IEEE Transactions on Automatic Control}, 55(2): 367--377, 2010.

\bibitem{lucas}
B. D. Lucas and T. Kanade.
\newblock An iterative image registration technique with an application to stereo vision.
\newblock In {\em Proceedings of the International Joint Conference on Artificial Intelligence}, pages 674--679, 1981.

\bibitem{2012_Mahony.IJC}
R. Mahony, T. Hamel, P. Morin, and E. Malis.
\newblock Nonlinear complementary filters on the special linear group.
\newblock {\em International Journal of Control}, 85(10): 1557--1573, 2012.

\bibitem{MahonyNOLCOS13}
R. Mahony, J. Trumpf and T. Hamel.
\newblock Observers for Kinematic Systems with Symmetry.
\newblock In {\em 9th IFAC Symposium on Nonlinear Control Systems}, pages 617--633, 2013.

\bibitem{MalChaBou1999}
E. Malis and F. Chaumette and S. Boudet
\newblock 2-1/2-D visual servoing.
\newblock {\em IEEE Transactions on Robotics and Automation}, 15(2): 238--250, 1999.

\bibitem{2009_Malis.icra}
E. Malis, T. Hamel, R. Mahony, and P. Morin.
\newblock Dynamic estimation of homography transformations on the special linear group for visual servo control.
\newblock In {\em Proceedings of the IEEE International Conference on Robotics and Automation (ICRA)}, pages 1498--1503, 2009.

\bibitem{MaVa2007}
E. Malis and M. Vargas,
\newblock Deeper understanding of the homography decomposition for vision-based control.
\newblock INRIA, Tech. Rep. 6303, 2007, available at \url{http://hal.inria.fr/inria-00174036/fr/}.

\bibitem{Mei2008}
C. Mei, S. Benhimane, E. Malis, and P. Rives.
\newblock  Efficient homography-based tracking and 3-D reconstruction for single-viewpoint sensors.
\newblock {\em IEEE Transactions on Robotics}, 24(6): 1352--1364, 2008.

\bibitem{2010_Mondragon_icra}
I. F. Mondrag\'{o}n, P. Campoy, C. Mart\'{i}nez, and M. A. Olivares-M\'{e}ndez.
\newblock 3D pose estimation based on planar object tracking for UAVs control.
\newblock In {\em Proceedings of the IEEE International Conference on Robotics and Automation (ICRA)}, pages 35--41, 2010.

\bibitem{RM_2007_Mufti_dicta}
F. Mufti, R. Mahony, and J. Kim.
\newblock Super resolution of speed signs in video sequences.
\newblock In {\em Proceedings of Digital Image Computing: Techniques and Applications (DICTA)}, pages 278--285, 2007.

\bibitem{ScandaroliIROS2011}
G. G. Scandaroli, P. Morin, and G. Silveira.
\newblock A nonlinear observer approach for concurrent estimation of pose, IMU bias and camera-to-IMU rotation.
\newblock In {\em Proceedings of the IEEE/RSJ International Conference on Intelligent Robots and Systems (IROS)}, pages 3335--3341, 2011.

\bibitem{2008_Scaramuzza_TRO}
D. Scaramuzza and  R. Siegwart.
\newblock Appearance-guided monocular omnidirectional visual odometry for outdoor ground vehicles.
\newblock {\em IEEE Transaction on Robotics}, 24: 1015--1026, 2008.

\bibitem{tomasi}
C. Tomasi and T. Kanade.
\newblock Detection and tracking of point features.
\newblock Carnegie Mellon University Technical Report CMU-CS-91-132, April 1991.

\bibitem{2005_Wang_iros}
H. Wang, K. Yuan, W. Zou, and Q. Zhou.
\newblock Visual odometry based on locally planar ground assumption.
\newblock In {\em Proceedings of the IEEE International Conference on Information Acquisition}, pages 59--64, 2005.

\bibitem{weng}
J. Weng, N. Ahuja, and T. S. Huang.
\newblock Motion and structure from point correspondences with error estimation: planar surfaces.
\newblock {\em IEEE Transactions on Pattern Analysis and Machine Intelligence}, 39(12): 2691--2717, 1991.

\bibitem{2007_Zhang_PR}
B. Zhang, Y. Li, and Y. Wu.
\newblock Self-recalibration of a structured light system via plane-based homography.
\newblock {\em Pattern Recognition}, 40: 1368--1377, 2007.

\bibitem{zhang}
Z. Zhang and A. R. Hanson.
\newblock 3D reconstruction based on homography mapping.
\newblock In {\em Proceedings of the ARPA Image Understanding Workshop}, pages 1007--1012, 1996.

\end{thebibliography}
\end{document}